\documentclass[preprint]{article}

    \usepackage[nonatbib]{neurips_2019}

\usepackage[utf8]{inputenc} 
\usepackage[T1]{fontenc}    
\usepackage{hyperref}       
\usepackage{url}            
\usepackage{booktabs}       
\usepackage{amsfonts}       
\usepackage{nicefrac}       
\usepackage{microtype}      
\usepackage[cmex10]{amsmath}
\usepackage{amsthm}
\usepackage{graphicx}
\graphicspath{{figs/}}
\usepackage[numbers]{natbib}
\usepackage{float}
\usepackage{footnote}
\makesavenoteenv{tabular}
\makesavenoteenv{table}

\usepackage[skip=0.8\baselineskip]{caption}

\title{Associative Convolutional Layers \thanks{All our implementations using Pytorch, alongside the datasets of our results (train accuracy, test accuracy, train loss, test loss, train time, and test time for all the epochs), trained models for Pytorch, and detailed documentations for our codes are available in the supplementary materials at \href{https://github.com/hamedomidvar/associativeconv}{github.com/hamedomidvar/associativeconv}.}}

\newtheorem{theorem*}{\textbf{Theorem}}
\newtheorem{proposition*}{\textbf{Proposition}}

\newtheorem{lemma*}{\textbf{Lemma}}
\newtheorem{conjecture*}{\textbf{Conjecture}}

\newtheorem{corollary*}{Corollary}

\author{%
  Hamed Omidvar\thanks{Department of Electrical and Computer Engineering, 
  UC San Diego,  La Jolla, CA 92093} \\
  University of California, San Diego \\
  \texttt{homidvar@ucsd.edu} \\
   \And
   Vahideh Akhlaghi \thanks{Department of Computer Science and Engineering, 
  UC San Diego,  La Jolla, CA 92093} \\
  University of California, San Diego \\
  \texttt{vakhlagh@ucsd.edu} \\
   \And
   Massimo Franceschetti \footnotemark[2] \\
     University of California, San Diego \\
  \texttt{mfranceschetti@ucsd.edu} \\
  \And
   Rajesh K. Gupta\footnotemark[3] \\
     University of California, San Diego \\
  \texttt{gupta@ucsd.edu} \\
}

\begin{document}

\maketitle 

\begin{abstract}
Motivated by the necessity for parameter efficiency in distributed machine learning and AI-enabled edge devices, we provide a general and easy to implement method for significantly reducing the number of parameters of Convolutional Neural Networks (CNNs), during both the training and inference phases. We introduce a simple auxiliary neural network which can generate the convolutional filters of any CNN architecture from a low dimensional latent space. This  auxiliary neural network, which we call “Convolutional Slice Generator” (CSG), is unique to the network and provides the association between its convolutional layers. During the training of the CNN, instead of training the filters of the convolutional layers, only the parameters of the CSG and their corresponding ``code vectors'' are trained. This results in a significant reduction of the number of parameters due to the fact that the CNN can be fully represented using only the parameters of the CSG, the code vectors, the fully connected layers, and the architecture of the CNN. To show the capability of our method, we apply it to ResNet and DenseNet architectures, using the CIFAR-10 and ImageNet-1000 datasets without any hyper-parameter tuning. Experiments show that our approach, even when applied to already compressed and efficient CNNs such as DenseNet-BC, significantly reduces the number of network parameters. In two models based on DenseNet-BC with $\approx 2\times$ reduction in one of them we had a slight improvement in accuracy and in another one, with $\approx 2\times$ reduction the change in accuracy is negligible. In case of ResNet-56, $\approx 2.5\times$ reduction leads to an accuracy loss within $1\%$. When applying this approach to ResNet-18 on ImageNet-1000 dataset, we achieved a top-1 error that is $1.7\%$ better than the original network while having $1.5\times$ reduction in the number of parameters. In case of ResNet-50, our approach reduces the number of parameters to less than the number of parameters of ResNet-18, namely by $\approx 1.7\times$, while the top-1 error degradation is less than $1\%$ compared to the original ResNet-50.

\end{abstract}

%
\section{Introduction}\label{Sec:Introduction}

\looseness -1
Current state-of-the-art Convolutional Neural Networks (CNNs) consist of hundreds or even thousands of convolutional layers \citep{he2016identity, pmlr-v80-xiao18a} and the resulting large number  of parameters presents a limit to their wider application.
More efficient implementations are desired, for training and inference phases of large CNNs running on the cloud. Also, on the other end of the spectrum,
as these networks proliferate to small embedded devices at the edge of the internet, and closer to observation and control in real-life applications, their size and implementation efficiency becomes critical. 
%


For the training of very large CNNs, for instance those running on cloud computing resources, distributed machine learning approaches are typically used. In this case, communication constraints present a key challenge, as gradients of the network parameters need to be communicated among different nodes \cite{wang2018atomo}. Federated learning is an example of distributed machine learning where a neural network is optimized and customized in a distributed manner using numerous users' edge devices \cite{konevcny2016federated}. 
Recent works have focused on coding, quantization, and compression techniques to reduce the amount of data that needs to be communicated among the different nodes~\citep{chen2018adacomp, wang2018atomo, lin2018deep}.

To improve the inference time of CNN models, especially on edge devices, recent studies propose various techniques such as pruning the network parameters and connections \citep{han2015learning,li2016pruning, anwar2017structured}.
Although these techniques overcome limited storage capacities in edge devices and reduce the number and cost of operations, they are only applicable to the models after the training phase. In addition, to recover the accuracy degradation resulting from these methods, extra training and fine-tuning is required.

Motivated by the challenges described above, in this paper  we focus on reducing the redundancy in the parameters describing the convolutional layers of CNNs and provide an approach which can be used in conjunction with all of the aforementioned solutions and is applicable  during both the training and inference phases.
Our contribution stems from the observation that although there has been considerable progress in more efficient implementation of neural networks,  large convolutional filters are always needed. Such convolutional filters are inherently redundant, to a point that pruning \cite{li2016pruning}, quantization \cite{hubara2017quantized,Flexpoint, bruna2013invariant},
and low rank approximations on these filters \cite{jaderberg2014speeding}
can be performed. This redundancy suggests that these layers could be represented in a much smaller space than their natural tensors space. Our approach is particularly relevant  in view of the recent trend of adding additional convolutional layers \cite{he2016identity, pmlr-v80-xiao18a} or adding additional filters to the convolutional layers \cite{zagoruyko2016wide} to achieve higher accuracy.

We provide a method to obtain a low-dimensional representation of the parameter space of the set of filters of convolutional layers during both the training and inference phases by introducing an auxiliary neural network that can be used alongside any CNN architecture. This auxiliary neural network generates slices of sets of convolutional filters and is called Convolutional Slice Generator (CSG).
The CSG takes as input a set of code vectors corresponding to a partition of a set of convolutional filters of each layer in a latent, low-dimensional space, and produces these slices, which are then combined and used as the set of convolutional filters of that layer. 

\looseness -1
The  code vectors, which lie in a space of cardinality  $\approx 20 \times$ smaller than the cardinality of the corresponding slice of the convolutional filter, are optimized during the training of the CNN instead of the set of filters. The auxiliary neural network can either be trained alongside the main network or be provided to the network in advance with pre-trained and fixed parameters. In our experiments on classification tasks, we show that while this approach significantly reduces the cardinality of the parameter space of the CNN, the resulting networks, except in extreme compression cases, still achieve top-1 accuracies that are within one percent of the original CNNs, or even achieve improved accuracies.  

Finally, one could argue that 
in this work we trade computation efficiency for parameter efficiency, and hence communication and storage efficiency. However, we also show that the added computational cost is negligible in practice, and with customized hardware for edge devices, our approach is also expected to improve timing performance.

\subsection{Related Works}
\looseness -1
There are several works, mostly in the intersection of signal processing and computer vision, that focus on the design of convolutional filters. For instance, in \cite{jacobsen2016structured} the authors, inspired by scattering networks \cite{sifre2013rotation, bruna2013invariant, mallat2012group}, introduce a structured method based on the family of Gaussian filters and its smooth derivatives, to produce the CNN filters from some basis functions that are also learned during the training phase. Steerable filter design  is another approach that has also been studied for about two decades   \cite{freeman1991design}. 

Closer to our approach is the design of low rank and separable filters. In this case, the main goal has been of achieving better computing performance \cite{tai2015convolutional, mamalet2012simplifying, jaderberg2014speeding}. For example, the work in \cite{rigamonti2013learning} shows that multiple image filters can be approximated by a shared set of separable (rank-1) filters, allowing large speedups with minimal loss in accuracy.
Other works have exploited the computing efficiency of Fast-Fourier-Transform (FFT) based multiplications \cite{abtahi2018accelerating, ding2017c}. These schemes require complex multiplications and efficient implementations of FFT. There are also methods based on the Winograd algorithm \cite{winograd1980arithmetic} for performing efficient convolutions in the real domain \cite{lavin2016fast}. 
All of these methods can be applied in conjunction with our approach to compress and accelerate the operations in the fully connected layer(s) or to accelerate the convolution operations. 


Additional works are concerned with methods to perform different stages of the training in parallel, or to reduce the amount of information that needs to be communicated between different nodes of the distributed computation network using compression, or quantizing the gradients \cite{wang2018atomo, lin2018deep, pmlr-v80-ye18a, li2014scaling, recht2011hogwild, NIPS2018_7405}. However, these works are not concerned with the architecture of the network or on how the filters are designed, and can be applied to any architecture including our CSG-augmented CNNs.

\subsection{Our Contribution}
We present three distinct contributions:
\begin{itemize}
    \item We provide a novel and general method for reducing the number of parameters that are needed to represent the sets of filters of convolutional layers during both the training and inference phases, through the use of an auxiliary neural network which transforms a set of code vectors in a low dimensional space to slices of sets of convolutional filters. The software implementation of our method is straightforward and it can be done by adding only a few lines of codes to the implementation of any CNN architecture. 
    \item We provide an example of a simple CSG-augmented CNN and show that the training time for this network is polynomial in the number of data points, number of input features (e.g., pixels), and inverse of the minimum distance between data points. In addition to this analysis, inspired by Discrete Cosine Transform (DCT)-based compression techniques for images, we provide an estimate on the relationship between the size of the slices and the cardinality of the code vector space which eliminates the need for tuning for these parameters. This analysis also suggests that our approach can be applied to at least a large set of architectures.
    \item We experimentally investigate the performance of our method by applying it to ResNet and DenseNet architectures, and show that significant parameter reductions, without compromising the accuracy, are possible. Furthermore, when running on a single GPU, we observe that the training time and the inference time of the augmented networks remain almost unaltered. 
\end{itemize}

The paper is organized as follows. In Section~\ref{Sec:Prelims} we provide the preliminaries and set the stage for introducing our method. In Section~\ref{Sec:CSG} we formally introduce the CSG, provide a rough estimate on the cardinality of the code vector space, and theoretically investigate its effect on the convergence of the training phase. In Section~\ref{Sec:Experiments} we provide the results of our experiments on ResNet and DenseNet architectures. 
Finally, Section~\ref{Sec:Conclusion} includes our concluding remarks and future directions.

\section{Preliminaries}
\label{Sec:Prelims}

\subsection{Convolutional Neural Network (CNN)}
In a typical classification task, a CNN is   composed of several convolutional layers and one or more fully connected layers, at the very end of the   network, responsible for the classification. Each convolutional layer consists of a set of filters (i.e., kernels) and perhaps some batch normalization layers and ReLu activations.  Here, we focus on the sets of filters of the convolutional layers.

Let $k\in \mathbb{R}^{s_1s_2s_3s_4}$, for $s_1,s_2,s_3,s_4 \in \mathbb{N}$, denote a set of $s_1$ filters in the CNN, where $s_2$ denotes the number of channels and $s_3$ and $s_4$ denote the height and width of the kernel, respectively. We denote the collection of all the sets of filters in a CNN by $\mathcal{K}$. Let $\mathcal{O}$ denote the set of all the other parameters in the CNN. We denote the set of all the parameters by $\mathcal{P} := \mathcal{K} \cup \mathcal{O}$.

\subsubsection{Slices}
We define a \textit{slice} as a tensor  $\hat{k} \in \mathbb{R}^{\hat{s}_1\hat{s}_2\hat{s}_3\hat{s}_4}$, for $\hat{s}_1,\hat{s}_2,\hat{s}_3,\hat{s}_4\in \mathbb{N}$. We partition each set of filters $k\in\mathcal{K}\setminus \{k_0\}$, where $k_0$ denotes the set of filters of the first convolutional layer, into $\lceil s_1/\hat{s}_1 \rceil \lceil s_2/\hat{s}_2 \rceil \lceil s_3/\hat{s}_3 \rceil \lceil s_4/\hat{s}_4 \rceil$ slices and denote the set of all these slices by $\hat{\mathcal{K}}$. For simplicity we assume that this partitioning is possible. \footnote{In practice we consider additional slices for fractional partitions and only use part of the final slice(s) to reconstruct the set of convolutional filters.}

\subsubsection{Code Vectors}
Let $c\in \mathbb{R}^{n_c}$, where $n_c\in \mathbb{N}$, denotes a vector of $n_c$ elements. We refer to $c$ as a \textit{code vector}. Each slice of each filter $\hat{k} \in \hat{\mathcal{K}}$ in the CNN corresponds to one code vector and their relationship is illustrated  in the following section.

\section{The Convolutional Slice Generator}
\label{Sec:CSG}

The Convolutional Slice Generator (CSG) is the core element of our approach. The CSG  provides a linear approximation for slices of a convolutional filter. 

\subsection{The CSG Network}

Let $vec(k)$ denote the vectorized version of a tensor $k$. Let $k_{i}$, for $i\in \{1,...,|\hat{\mathcal{K}}| \}$ denote a slice in $\hat{\mathcal{K}}$ and $c_i$ for $i\in \{1,...,|\hat{\mathcal{K}}| \}$ denote the code vector corresponding to the $i$'th slice in $\hat{\mathcal{K}}$ 
\begin{align}
vec(\hat{k}_{i}) = A_{CSG}c_{i} \mbox{, \ \ \  for \ \ \   } i\in \{1,...,|\hat{\mathcal{K}}|\},
\end{align}
where $A_{CSG}$ denotes an $\hat{s}_1\hat{s}_2\hat{s}_3\hat{s}_4$ by $n_c$ matrix representing the weights of the CSG network, $c_i$ denotes the code vector corresponding to the $i$'th slice where $i \in \{1,2,...,|\hat{\mathcal{K}}|\}$. See Fig.~\ref{Fig:CSG} for an example of how a single slice of a set of filters for a single convolutional layer is generated.

Let $\hat{\mathcal{G}}$ denote all the parameters of the CSG, i.e., the elements of the matrix $A_{CSG}$, $\hat{\mathcal{C}}$ denote the set of all the code vectors, and let $\hat{\mathcal{O}}$ denote all the parameters of the CNN except for the parameters in $\hat{\mathcal{G}}$ or $\hat{\mathcal{C}}$, e.g., biases, batch normalization parameters, fully connected layer(s), and the first convolutional filter. Hence, we can denote the set of all the parameters of the network by $\hat{\mathcal{P}} := \hat{\mathcal{C}} \cup \hat{\mathcal{G}} \cup  \hat{\mathcal{O}}$. 

\begin{figure}
  \centering
  \includegraphics[scale=.42]{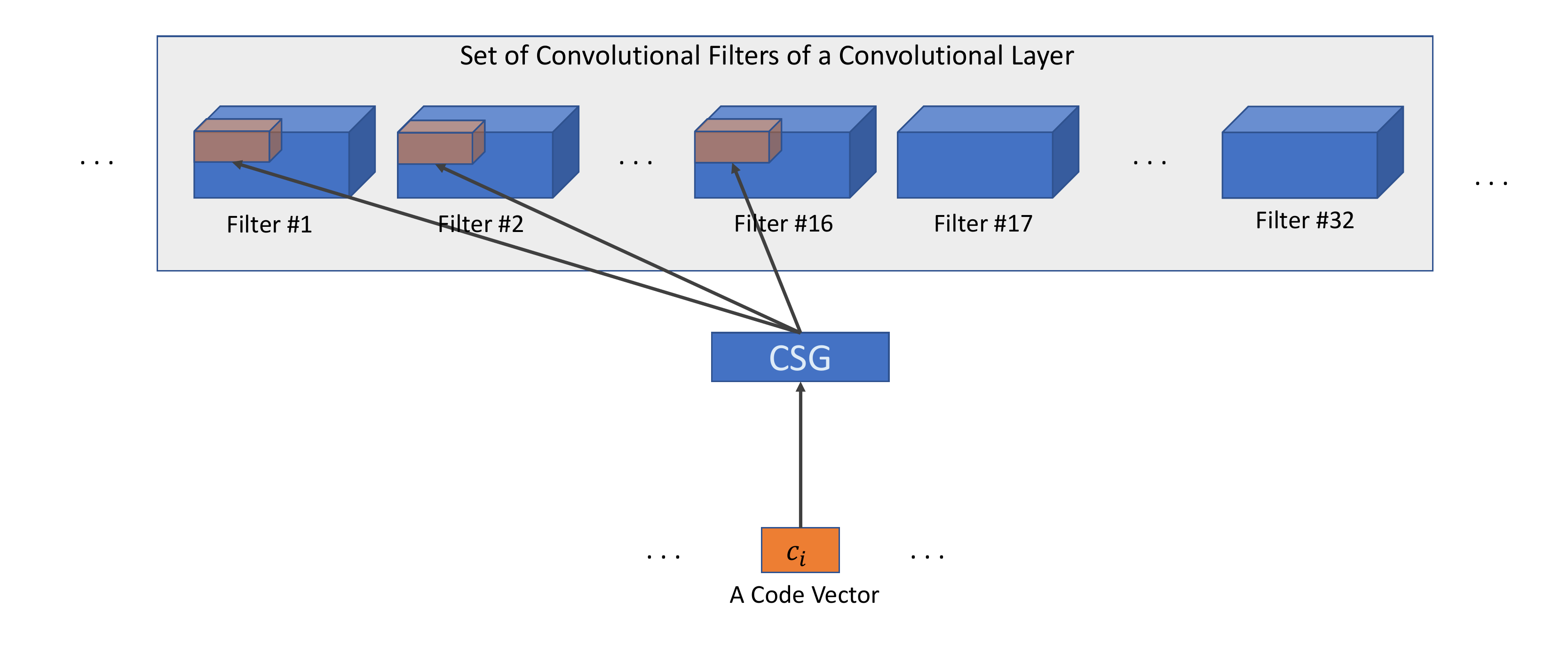}
  \caption{Generation of a single slice of a set of convolutional filters of a convolutional layer. Our method can be applied to any filter shape. In this example there are 32 filters, kernels are 6$\times$6 and have 32 channels. Each slice, generated by the CSG, is assumed to be $16\times 16 \times3\times 3$. The figure shows one slice, that spans across multiple channels and multiple filters, and its corresponding code vector. }\label{Fig:CSG}
\end{figure}

\subsection{Estimating the Cardinality of the Code Vector Space}\label{Sec:Estimate}
In this section, we discuss our method for having a very rough estimate on the cardinality of the code vector space $n_c$. First, we need to choose a shape for the slices. 
In order to decide about this shape, we considered several widely used CNNs including VGG16, VGG19, ResNet, etc., and concluded that a $3\times 3$ filter size is the most common size for the filters. Also, these architectures suggests that a slice with channel size of $16$ and the depth of $16$ would divide most of these filters. Hence, we chose $\hat{s}_1 = 16, \hat{s}_2 = 16, \hat{s}_3 = 3, \hat{s}_4 = 3$ for this part of our work. 

In order to decide about the cardinality of the vector space, we need an estimate on the number of the elements of the slice in its possible latent domain, namely an estimate for $n_c$. Inspired by the fact that these filters are responsible for detecting visual features and knowing that usage of DCT leads to a very good encoding of visual representations \cite{watson1994image}, we looked at the four-dimensional Type-II DCTs (4-D DCT-II) of about 29000 slices of pre-trained filters extracted from VGG-16, VGG-19, ResNet-50, InceptionV3, DenseNet-169, DenseNet-201, InceptionResNetV2 (available in Tensorflow). We then computed the 4-D DCT-II representation of these slices and removed the elements of this representation in such a way that the remaining elements would result in an inverse transform which is not very different from the original slice. Our analysis, presented in Appendix~\ref{App:Estimate}, suggests that a code vector that has close to $20\times$ fewer number of elements would be sufficient. In our experiments, we chose code vectors that have $18\times$ fewer elements than the slices, and our experiments on the neural networks confirm this choice.

\subsection{Training Convergence}
While  convergence is always observed in all our experiments,
in this section, we provide a proof of convergence for a simple CNN with only one convolutional layer based on the recent work \cite{allen2018convergence}. 

Let $m$ denote the number of channels of the input, and $d$ denote the number of its features (e.g., pixels). For simplicity, let us assume that the number of channels remains $m$ after the convolutional layer. Let $n$ denote the number of data points, and $d'$ denote the number of labels. We assume that the data-set is non-degenerate meaning that there does not exist similar inputs with dissimilar labels. We denote by $\delta$ the minimum distance between two training points. We restate the following theorem from \cite{allen2018convergence} for the CNN defined in Appendix~B of this reference.

\begin{theorem*}[CNN \cite{allen2018convergence}] \label{Thrm:cnn}
As long as $m\ge \Tilde{\Omega}(poly(n,d,\delta^{-1})d')$, with a probability that approaches one as $m\rightarrow{\infty}$, Stochastic Gradient Decent (SGD) finds an $\epsilon$-error solution for $l_2$ regression in $T=\Tilde{\Omega}\left( \frac{poly(n,d)}{\delta^2}\log\epsilon^{-1} \right)$ iterations for a CNN.
\end{theorem*}

\looseness -1
The above theorem as discussed in \cite{allen2018convergence} can be easily extended for other convergence criteria including the cross-entropy. Now let us consider our CSG-augmented CNN which we denote by CNN-CSG. For simplicity, in the following theorem, we consider the case when only a single layer convolutional layer is present. 

\begin{theorem*}[CNN-CSG]\label{Thrm:cnn-CSG}
If $|\hat{\mathcal{C}}| \ge \Tilde{\Omega}(poly(n,d,\delta^{-1})d')$, with a probability that approaches one as $|\hat{\mathcal{C}}|\rightarrow{\infty}$, then SGD finds an $\epsilon$-error solution for $l_2$ regression in $T=\Tilde{\Omega}\left( \frac{poly(n,d)}{\delta^2}\log\epsilon^{-1} \right)$ iterations for a CNN-CSG.
\end{theorem*}

\looseness -1
The proof of the above theorem, which follows from the fact that the code vectors following the CSG layer can simply be viewed as an additional fully connected layer, can be found in Appendix~\ref{App:Proof2}. Similar to Theorem~\ref{Thrm:cnn},  Theorem~\ref{Thrm:cnn-CSG} can be easily extended for other convergence criteria including the cross-entropy. 

\section{Experiments} 
\label{Sec:Experiments}

\subsection{Setup}
\looseness -1
We evaluated our approach on three different CNN models (ResNet56, DenseNet-BC-40-48, DenseNet-BC-40-36) on CIFAR-10 dataset. CIFAR-10 includes 50K training images and 10K test images from 10 different classes. The CSGs are integrated into the models implemented in Pytorch. Our implementations along with detailed documentations of our codes are available in the supplementary materials.
For training the models, we used a machine with a single GPU (Nvidia Geforce 2080 Ti). It is worth mentioning that we did not do any parameter tuning for our CSG-augmented networks and the experiments are all done using the same settings that we used for the original networks. Also, as it is clear from the previous sections, we did not apply our method to the very first convolutional layer of any network.

\subsection{Training CSG alongside the CNN} 
\looseness -1
In this set of experiments we train all the models from scratch. We initialize the parameters of CSG $\hat{\mathcal{G}}$ with random initial values and train it alongside the code vectors $\hat{\mathcal{C}}$ as well as other parameters of the network $\hat{\mathcal{O}}$. 
\subsubsection{CIFAR-10 Dataset}

See Table~\ref{Tbl:Results} for a summary of the results. As we can see, when we used $[16,16,3,3]$ slices and code vectors of size $128$ for ResNet-56 \cite{he2016identity}, we achieved  $\approx 2.5\times$ reduction with less than $1\%$ increase in top-1 error. If we allow a higher accuracy degradation of $\approx 1.5\%$, we can achieve over $5.3\times$ parameter reduction by using $[12,12,3,3]$ slices and code vectors of size $72$. In case of DenseNet \cite{huang2017densely}, we considered the most challenging cases, namely, when bottlenecks are used and the network has a $50\%$ compression factor (i.e., $\theta=0.5$), which is abbreviated as DenseNet-BC. We only considered $3\times 3$ kernels and did not compress the bottleneck or transition layers in these implementations. Since the number of filters is a multiple of $12$, we chose slices of shape $[12,12,3,3]$ and code size of $72$ to keep the ratio between the number of elements in the slice and $n_c$ the same. We considered two cases when $L=40, K=48$, and $L=40, K=36$, where $L$ is the number of layers and $K$ is the growth rate. For the first case, we could achieve $\approx 2\times$ reduction with a slight improvement in accuracy. For the second case, the use of CSG had little effect on the accuracy of the network while reducing its parameters by over $1.8\times$.
\begin{table}
\vspace{-1mm}
  \caption{Training results on CIFAR-10 dataset. When CSG is used, the slice shape and the code vector size are indicated as CSG-$[\hat{s}_1,\hat{s}_2,\hat{s}_3,\hat{s}_4]$-$n_c$ following the name of the original network. In the ``Top-1 Err.'' column the average and standard deviations of test errors at the last epoch for three non-selective trainings and on the ``Ratio'' column the compression ratios with respect to the original networks are reported.}
\label{Tbl:Results}
  \centering
  \begin{tabular}{lccc}
    \toprule
    \cmidrule(r){1-2}
    Network Architecture  & $\#$ Param.     &  Top-1 Err. &  Ratio  \\
    \midrule
    DenseNet-BC-40-48 (Original) & 2,733,130  & 4.97 $\pm$ 0.26 & 1.00$\times$   \\
    DenseNet-BC-40-48-CSG-[12,12,3,3]-72  & 1,416,394  & 4.83 $\pm$ 0.24  & 1.92$\times$  \\
    
    DenseNet-BC-40-48-CSG-[12,12,3,3]-72  \\
     \ \ \ \ \  w/ Pre-trained CSG on DenseNet-BC-40-48 & 1,323,082  & 5.07 $\pm$ 0.11 & 2.06$\times$ \\
     DenseNet-BC-40-48-CSG-[12,12,3,3]-72\\
     \ \ \ \ \  w/ Pre-trained CSG on DenseNet-BC-40-36  & 1,323,082  & 5.14 $\pm$ 0.23 & 2.06$\times$ \\
    DenseNet-BC-40-48-CSG-[12,12,3,3]-72 \\
    \ \ \ \ \ w/ Compressed 1x1 Kernels & 904,906  & 5.62 $\pm$ 0.28  & 3.02$\times$  \\
    \midrule
    DenseNet-BC-40-36 (Original) & 1,542,682 & 5.38 $\pm$ 0.27  & 1.00$\times$  \\
    DenseNet-BC-40-36-CSG-[12,12,3,3]-72 & 842,842  & 5.12 $\pm$ 0.09 & 1.83$\times$ \\
    DenseNet-BC-40-36-CSG-[12,12,3,3]-72  \\
    \ \ \ \ \  w/ Pre-trained CSG on DenseNet-BC-40-48 & 749,530  & 5.61 $\pm$ 0.21  &  2.05$\times$\\
\midrule
\midrule
 ResNet-56 (Original) & 853,018  & 6.28 $\pm$ 0.20 & 1.00$\times$  \\
    ResNet-56-CSG-[16,16,3,3]-128     & 347,162 & 7.26 $\pm$ 0.19  & 2.45$\times$   \\
    ResNet-56-CSG-[12,12,3,3]-72     & 160,450 & 8.01 $\pm$ 0.27     & 5.31$\times$ \\
        ResNet-56-CSG-[16,16,3,3]-128 \\ 
    \ \ \ \ \ w/ Pre-trained CSG on ResNet-20     & 52,250       & 11.98 $\pm$ 0.28 & 16.3$\times$ \\    
    
    \bottomrule
  \end{tabular}
  \vspace{-4mm}
\end{table}

\begin{figure}
  \centering
  \includegraphics[width=0.98\textwidth]{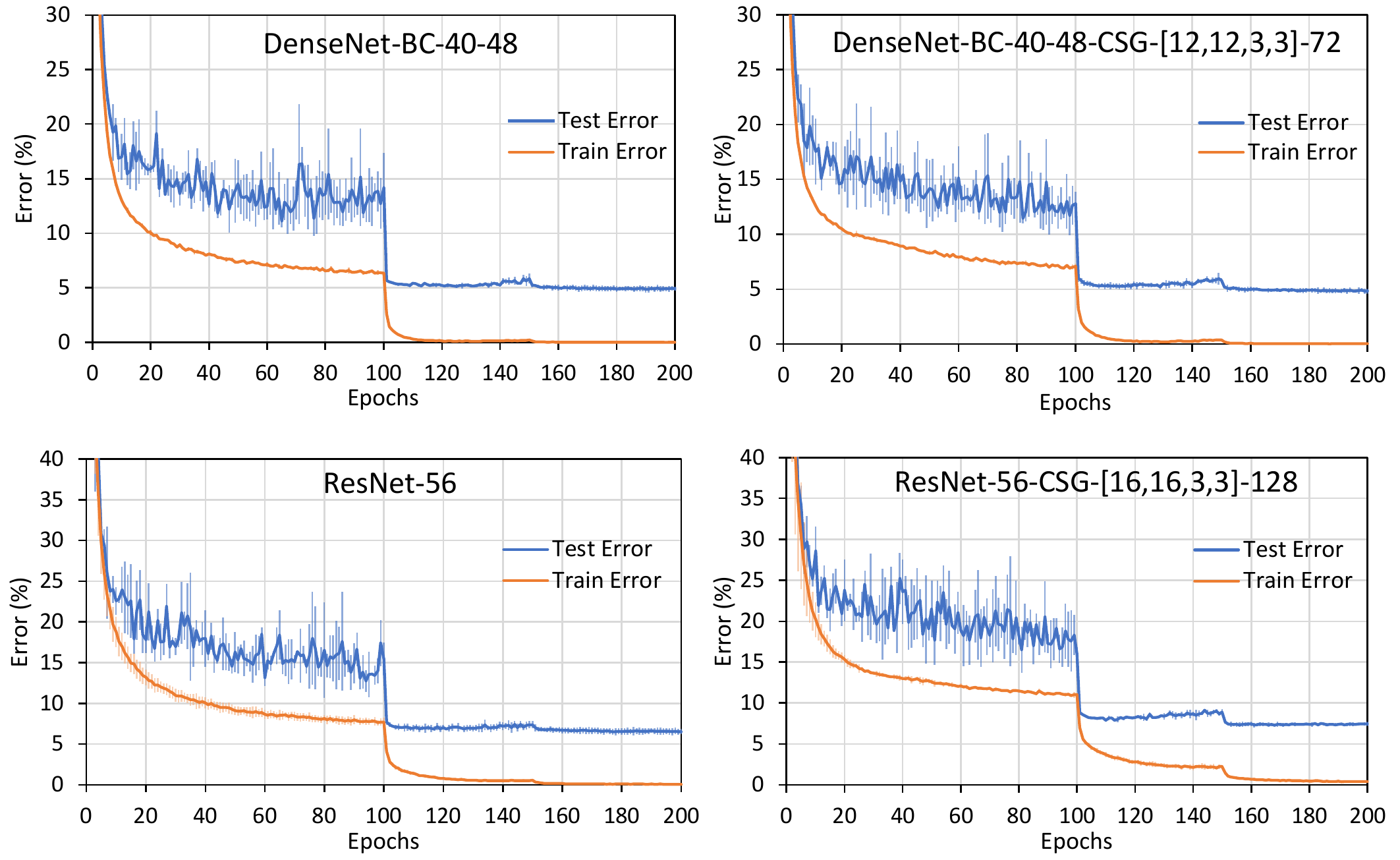} 
  \caption{Training and test error for DenseNet-BC-40-48 and ResNet-56 and their CSG-augmented versions on CIFAR-10 dataset over the course of 200 epochs. For the first 100 epochs the learning rate was set to $0.05$ and for the two final $50$ epochs it was set to $5 \times 10^{-3}$ and $5\times 10^{-4}$ respectively, and the batch size was $128$ for all DenseNet models and $192$ for all ResNet models.}\label{Fig:ResDense-errorbars}
  \vspace{-5mm}
\end{figure}

\begin{table}
  \caption{Training results on ImageNet-1000 (ILSVRC2012) dataset. When CSG is used, the slice shape and the code vector size are indicated as CSG-$[\hat{s}_1,\hat{s}_2,\hat{s}_3,\hat{s}_4]$-$n_c$ following the name of the original network. In the ``Top-1 Error'' column the validation error for the center cropped images at the last epoch for the training and on the ``Ratio'' column the compression ratios with respect to the original networks are reported. The results indicated with a "*" are reported from \cite{he2016deep}.}
\label{Tbl:ResultsImagenet}
  \centering
  \begin{tabular}{lccc}
    \toprule
    \cmidrule(r){1-2}
    Network Architecture  & $\#$ Param.     &  Top-1 Err. ($\%$) &  Ratio  \\
    \midrule
    ResNet-18 (Original) & 15,995,176  & 30.2\%*   & 1.00$\times$   \\
    ResNet-18-CSG-[16,16,3,3]-128  &  10,371,368   & \textbf{28.5}\%  & 1.54$\times$  \\
    \midrule
    ResNet-50 (Original) & 25,557,032  & 24.7\%*  & 1.00$\times$   \\
    ResNet-50-CSG-[16,16,3,3]-128  & 15,163,432  & 24.9\%  & 1.68$\times$  \\

    \bottomrule
  \end{tabular}
  \vspace{-4mm}
\end{table}

\begin{figure}
\vspace{-3mm}
  \centering
  \includegraphics[width=0.9\columnwidth]{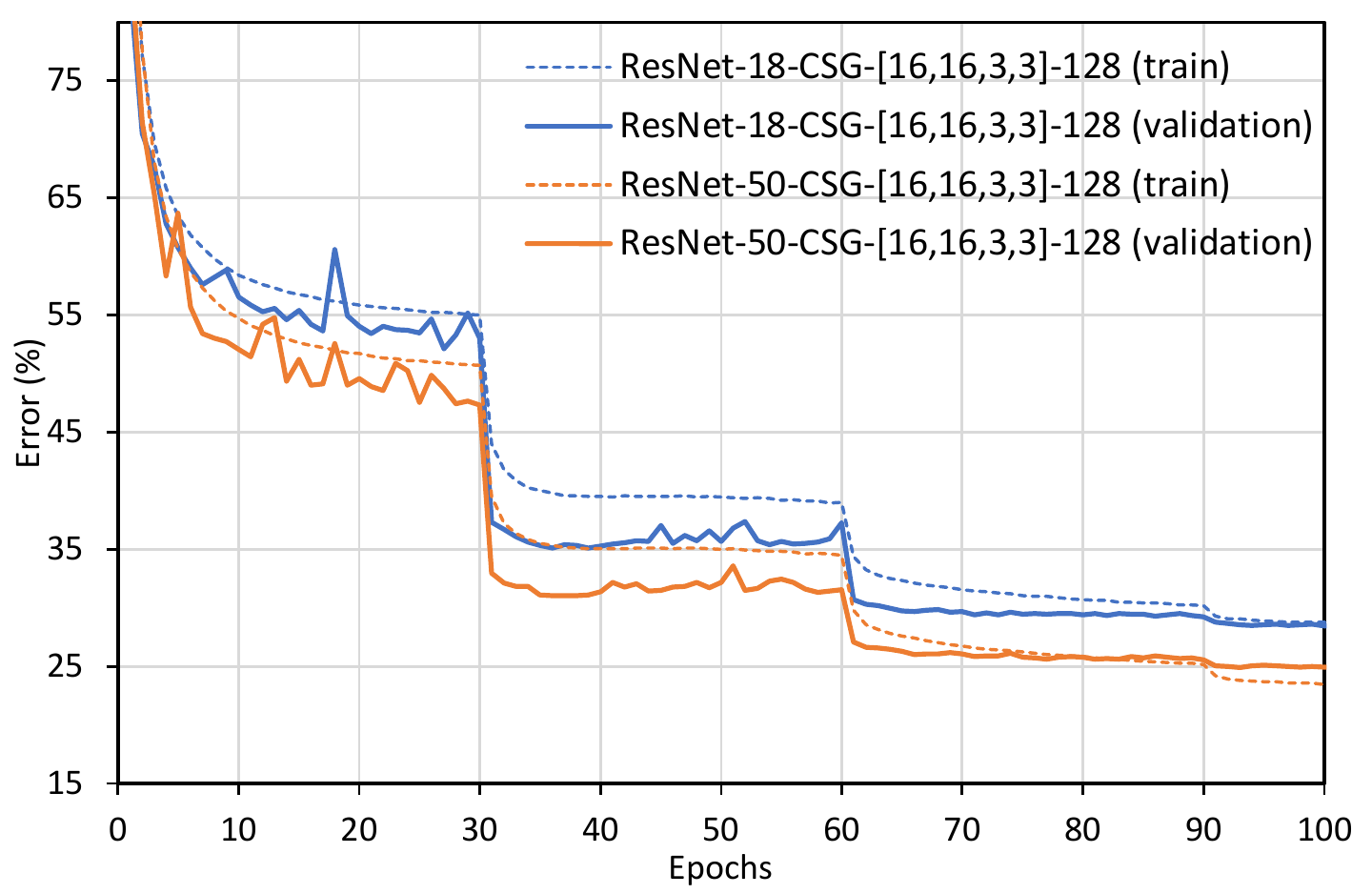} 
  \caption{Train and validation errors during the training of ResNet-18-CSG-[16,16,3,3]-128, and ResNet-50-CSG-[16,16,3,3]-128 on ImageNet dataset.}\label{Fig:resnet-imagenet}
  \vspace{-3mm}
\end{figure}


\subsubsection{ImageNet-1000 (ILSVRC2012) Dataset}
We have also trained the CSG-augmented versions of ResNet-18 and ResNet-50 on the ImageNet-1000 (ILSVRC2012) dataset. We used the same hyperparameters as the ones mentioned in the original paper \cite{he2016deep}, namely we used batch sizes of 256 images, and started from the learning rate of 0.1 and divided the learning rate by 10 every 30 epochs. We continued the training for 100 epochs which is 20 epochs fewer than the original paper. While ResNet-18-CSG-[16,16,3,3]-128 has a compression ratio of $1.54\times$, it achieves a top-1 error of $28.5 \%$ which is $1.7\%$ better than the implementation of the original ResNet-18 as reported in \cite{he2016deep}. ResNet-50-CSG-[16,16,3,3]-128 which has almost the same number of parameters as ResNet-18, achieves $24.9 \%$ top-1 error with a compression ration of $1.68\times$. The results are summarized in Table~\ref{Tbl:ResultsImagenet}. More details of training and validation errors over the course of 100 epochs are brough in Figure~\ref{Fig:resnet-imagenet}.

\subsection{Using Pre-Trained CSG}
\looseness -1
When using pre-trained CSG parameters during the training of the CSG-augmented CNNs, the number of parameters to be trained reduces to $|\hat{\mathcal{C}}|+|\hat{\mathcal{O}}|$. This can result in significant reduction in the number of the parameters of the network depending on its architecture. 
For ResNet-56, in the case of using fixed pre-trained parameters for the CSG that was trained alongside ResNet-20 architecture (in ResNet-20, due to the small size of the network, use of the CSG-augmented network does not result in parameter reduction, i.e., $\hat{\mathcal{G}}$ is larger than the number of parameters - training and test details of ResNet-20 are available in supplementary materials), the number of parameters reduces from about $850K$ to merely $50K$, a reduction of more than $16\times$ but at the cost of higher accuracy loss. For DenseNet-BC  when $L=40$,$K = 48$, our approach of using pre-trained CSG that was trained alongside DenseNet-BC with $L=40$, $K = 36$ reduces the number of parameters from $\approx 2.7$ million to $\approx 1.3$ million (i.e., $2.06\times$) while also improving the accuracy. For DenseNet-BC  when $L=40$, $K = 36$, this approach (using a pre-trained CSG obtained from DenseNet-BC with $L=40$, $K = 48$) reduces the number of parameters from $\approx 1.5$ million to $\approx 0.75$ million while having a degradation of less than $0.5\%$ in accuracy. 

\subsection{End-to-End Timings}
We evaluated the training and inference time of CSG-augmented CNNs on a single GPU for different CNN models, and compared it with the baseline ones. We measured the execution time of training and inference stages on the whole train and test datasets. The average epoch time for each network is summarized in Table \ref{Tbl:Timings}. The results show that for DenseNet, both inference time and training time are slightly improved in the CSG-augmented models compared to the baseline. For ResNet-56, execution times reported for the CSG-augmented model are slightly higher than the baseline model. 
The results indicate that although CSG is added to each convolutional layer in CSG-augmented CNNs, the execution time on a single GPU remains almost the same. The reason is that in CSG-augmented networks, due to the reduced number of parameters, costly memory accesses like DRAM accesses across memory hierarchy in a computing system are decreased; thus, the communication and memory accesses cost are reduced. Therefore, the cost of more computation performed in CSG-augmented CNN models (i.e., additional operations related to matrix multiplications in CSG) compared to the baseline models do not increase the execution time of the baseline model. This results in a slight improvement of timing in DenseNet models, however, in case of ResNet, due to its smaller size, the lower memory access cost does not fully compensate the additional cost due to the use of CSG.

\begin{table}
  \caption{Training time, inference time and model sizes (Mega Bytes (MB)) and host to device model data transfer. Training and inference time for each epoch are reported in the following format: ``mean'' $\pm$ ``standard deviation'' over all 200 epochs. Model size reports the size of the Pytorch model in MB. ``H to D'' column reports the amount of model related data transmitted from the host (main memory) to the device (GPU memory) collected using the `nvprof' tool.}
\label{Tbl:Timings}
  \centering
  \resizebox{\textwidth}{!}{
  \begin{tabular}{lcccc}
    \toprule
    \cmidrule(r){1-2}
    Network Architecture      &    Train Time & Test Time & Model Size & H to D \\
    \midrule
    ResNet-56 (Original)   & 14.51s $\pm$ 0.32s & 1.11s $\pm$ 0.02s & 3.33MB & 4.30MB \\
    ResNet-56-CSG-[16,16,3,3]-128  & 14.74s $\pm$ 0.27s  & 1.30s $\pm$ 0.03s & 1.40MB & 2.28MB \\
    ResNet-56-CSG-[12,12,3,3]-72  & 14.97s $\pm$ 0.29s   & 1.40s $\pm$ 0.06s & 0.67MB & 1.53MB \\ 
    \midrule
    DenseNet-BC-40-48 (Original)   & 68.30s $\pm$ 0.55s  & 4.40s $\pm$ 0.06s & 10.50MB & 11.97MB \\
    DenseNet-BC-40-48-CSG-[12,12,3,3]-72    & 67.53s $\pm$ 0.51s  & 4.33s $\pm$ 0.04s & 5.52MB & 6.69MB \\
    \midrule
    DenseNet-BC-40-36 (Original) & 51.87s $\pm$ 0.43s  & 3.33s $\pm$ 0.03s  & 6.34MB & 7.18MB \\
    DenseNet-BC-40-36-CSG-[12,12,3,3]-72   & 51.58s $\pm$ 0.40s & 3.31s $\pm$ 0.04s & 3.31MB & 4.38MB \\
    \bottomrule
  \end{tabular}}
  \vspace{-5mm}
 \end{table}

We expect the execution time to be improved considerably on specialized hardware architectures such as Application Specific Integrated Circuits (ASICs) and Field Programmable Gate Arrays (FPGAs), mostly used in edge devices, for the following reasons. First, in these hardware architectures, as shown in \cite{Diannao}, the required memory bandwidth for the model parameters are relatively high compared to other data such as input/output feature maps. In contrast, on GPU, because of data parallelization, the intermediate results such as feature maps for different input images consume a more significant part of the on-chip memories and off-chip memory bandwidth. 
Second, in specialized architectures designed for CNNs such as the ones introduced in \cite{eyeriss, Diannao}, the proposed mapping of the operations and data on the processing elements helps to increase the reusability of data and parameters loaded into the on-chip memories, which reduces the number of accesses to costly memories (i.e., off-chip memories). In these cases, with CSG-augmented CNNs, fewer number of parameters are loaded onto the on-chip memories such as global buffers in Eyeriss architecture \cite{eyeriss}, and their reuse distance is increased due to the fact that fewer number of parameters can be loaded onto on-chip memories to generate the same number of weights that must be loaded onto the on-chip memories in baseline models.

\looseness -1
In addition to evaluating the timing, we reported the size of the model related parameters that are transferred from the main memory of a computing machine (Host) to DRAM of the GPU (Device) (i.e., Host to Device (``H to D'' column in Table \ref{Tbl:Results})) for an inference task. The numbers, extracted from an Nvidia profiling tool (nvprof) show that with CSG-augmented CNNs, the communication between host and device memories to transfer the model is reduced by, on average, 2.03$\times$, compared to the baseline models.

\subsection{Comparison with Other Methods}
As mentioned before, our approach can be used on top of most of the available approaches. 
We have implemented two of the other compression methods mentioned in the above for DenseNet-BC-40-48: Separable filters \cite{rigamonti2013learning} and low rank filters using singular value decomposition (SVD) \cite{tai2015convolutional}. Separable convolutions approach reduced the accuracy by $\approx 1\%$ with $\approx 2 \times$ compression with roughly similar timings for training and inference on the GPU, 
while our approach slightly increases the accuracy with $\approx2\times$ compression. 
Parameter tuning for the low rank filters method (SVD-based) is needed and with moderate tuning, we have not been able to achieve better than $\approx 11\%$ top-1 error. Also, the training is $\approx 3\times$ slower due to need for computing the SVDs.
The number of trainable parameters is similar to the original model 
; however, after training, the decomposed parameters (reduced by $2\times$) are used for the inference. 
We will also provide a qualitative comparison between major approaches in a table in the final version which we could not include due to limited space.
 
\begin{table}[!htbp]
  \caption{Comparison of compression methods. }
\label{Tbl:Comparison}
  \centering
\resizebox{\textwidth}{!}{
  \begin{tabular}{lccccccccc}
    \toprule
    \cmidrule(r){1-2}
    Method & Implem. & Accuracy & Train Time & Inf. Time & Comp. Ratio & C. Layers & F.C. Layers & Dist. Learning \\ 
    \midrule
Ours & Easy & Almost Same & Almost Same & Faster (Mem. Access) & $\approx 2\times$ & Yes & No \footnote{Although we have not implemented this approach for approximating fully connected layers, the extension of our approach is straighforward and is left as a future direction.} & Yes \\
Low-rank & Difficult & Almost Same & Slower & Faster & $\approx 2\times$ & Yes & No & Yes \\
Separable & Easy & Slightly Degrades & Almost Same & Faster & $\approx 2\times$  & Yes & No & Yes \\
    \bottomrule
  \end{tabular}}

 \end{table}

\section{Conclusion and Future Directions}
\label{Sec:Conclusion}
Although several methods for making the convolutional layers of CNNs more efficient are used, the number of parameters of these layers constitute the most significant portion of the model parameters. In this work we focused on reducing the number of unnecessary parameters of convolutional layers by representing them in a low dimensional space through the use of a simple auxiliary neural network without significantly compromising the accuracy or tangibly adding to the processing burden.
There are still several directions that can be pursued. The use of this method for other tasks, especially other than vision related tasks, such as natural language processing, etc. needs to be assessed. The extension of the theoretical analysis to other more complicated architectures is an attractive future direction. The combination of this method with efficient computation, compression, and quantization methods mentioned in this paper for distributed machine learning and machine learning acceleration for edge devices are all worthwhile studies. Also the use of more than one CSG for different classes of filters or the use of non-linear and/or multi-layer CSGs should be investigated.


\small
\bibliographystyle{plainnat}
\bibliography{refs}
\raggedbottom
\pagebreak

\appendix
\section{Estimating the Cardinality of the Code Vector Space}\label{App:Estimate}
In this appendix we make the statements in Section~\ref{Sec:Estimate} more precise.  
We first take the 4-D DCT-II of each slice defined in Section~\ref{Sec:Estimate}.
The 4-D DCT-II that we use, after removing the scaling factors, is stated as follows.

\begin{align}
    &K[u,v,w,t] := \sum_{i=0}^{\hat{s}_1-1}\sum_{j=0}^{\hat{s}_2-1}\sum_{k=0}^{\hat{s}_3-1}\sum_{l=0}^{\hat{s}_4-1} k[i,j,k,l] \nonumber \\ 
    & \ \ \ \ \ \ \ \ \ \    \cos{\left(\frac{\pi}{\hat{s}_1}\left(i+\frac{1}{2}\right)u\right)}
    \cos{\left(\frac{\pi}{\hat{s}_2}\left(j+\frac{1}{2}\right)v\right)}
    \cos{\left(\frac{\pi}{\hat{s}_3}\left(k+\frac{1}{2}\right)w\right)}
    \cos{\left(\frac{\pi}{\hat{s}_4}\left(l+\frac{1}{2}\right)t\right)}
\end{align}

After taking the 4-D DCT-II, we then remov the elements of the slice in the transformed domain that were smaller than a threshold. We then took the inverse transform.
The inverse 4-D DCT transform, after neglecting its scaling factors, can be stated as follows.

\begin{align}
    &k[i,j,k,l] := \sum_{u=0}^{\hat{s}_1-1}\sum_{v=0}^{\hat{s}_2-1}\sum_{w=0}^{\hat{s}_3-1}\sum_{t=0}^{\hat{s}_4-1} K[u,v,w,t] \nonumber \\ 
    & \ \ \ \ \ \ \ \ \ \    \cos{\left(\frac{\pi}{\hat{s}_1}\left(u+\frac{1}{2}\right)i\right)}
    \cos{\left(\frac{\pi}{\hat{s}_2}\left(v+\frac{1}{2}\right)j\right)}
    \cos{\left(\frac{\pi}{\hat{s}_3}\left(w+\frac{1}{2}\right)k\right)}
    \cos{\left(\frac{\pi}{\hat{s}_4}\left(t+\frac{1}{2}\right)l\right)}
\end{align}

In order to measure the similarity between the inverse transformed version of the slice and the original slice, inspired by image compression similarity measures, we use a variation of a known measure called PSNR \cite{welstead1999fractal} which we define as follows. Let $\hat{k^*}$ denote the inverse DCT of the pruned DCT of the slice $\Hat{k}$. We re-scale the elements of the slices and their corresponding approximate version to $[0,1]$ with a bit of abuse of notation we represent the re-scaled versions with the same notations.  
\begin{align}
    PSNR^* = 10\log \frac{1^2}{MSE}, 
\end{align}
where
\begin{align}
    MSE = \frac{1}{\hat{s}_1\hat{s}_2\hat{s}_3\hat{s}_4}||\hat{k}-\hat{k}^*||^2_2.
\end{align}
We chose the threshold for keeping the elements in the DCT domain such that the average PSNR* is above $20$dB which from image compression literature is expected to result in images that are still recognizable (see, for instance \cite{veldhuizen2010measures}).
We then calculated the mean of the number of remaining elements in the DCT domain after the pruning step.
This suggests that a code size of 20 times fewer elements than its corresponding slice would be sufficient. Based on these estimates, in most of our experiments we choose code vectors whose number of elements is 18 times smaller than that of their corresponding slices.

\section{Training on CIFAR-10 Dataset}
The training and test error of different models for CIFAR-10 dataset and their CSG-augmented versions reported in Table \ref{Tbl:Results}, but not included in Figure \ref{Fig:ResDense-errorbars} are provided in figures \ref{Fig:densenet48-apdx}-\ref{Fig:resnet-apdx}.

\begin{figure}[H]
  \centering
  \includegraphics[scale=.6]{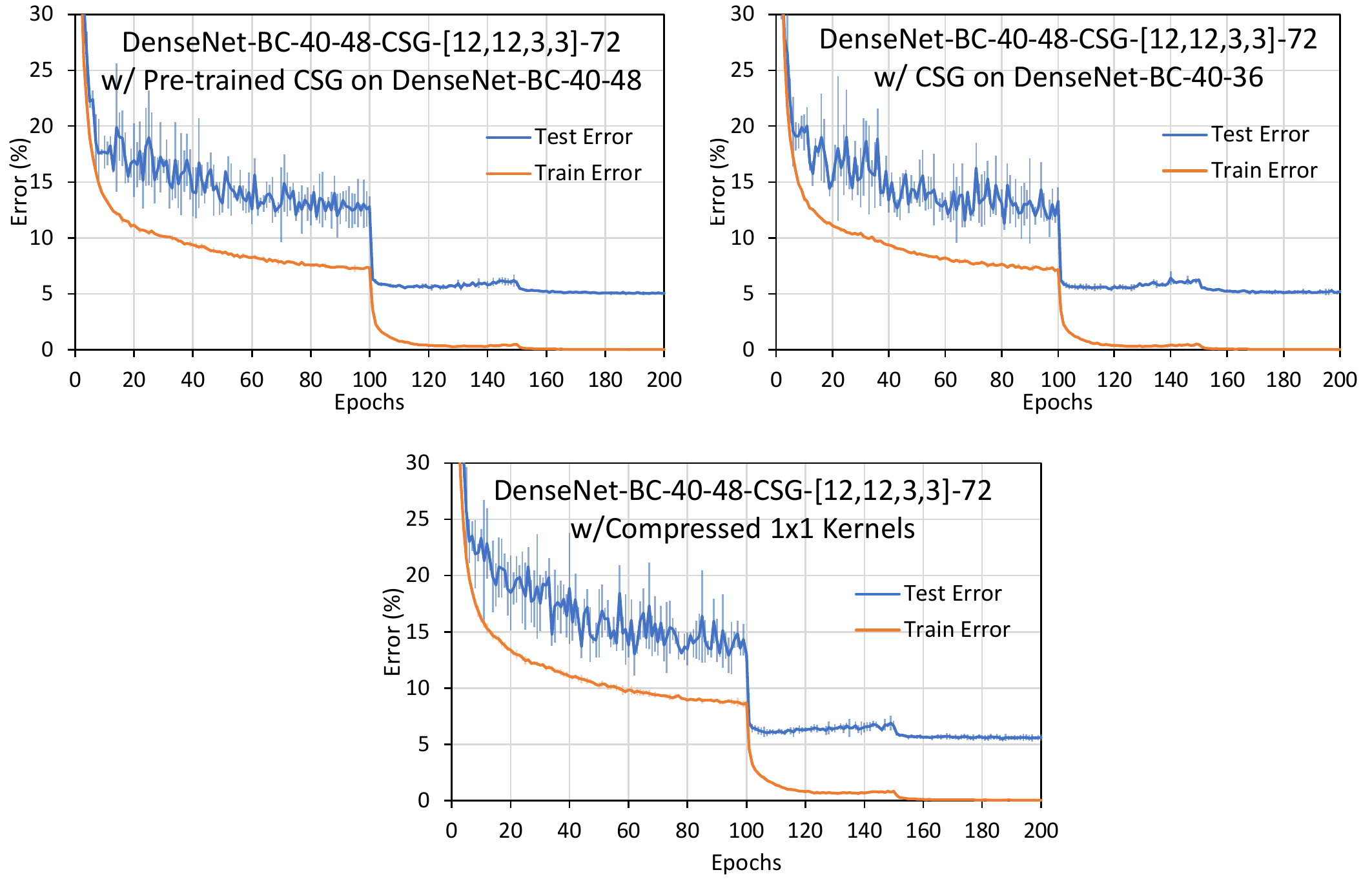}
  \caption{Training and test error for the CSG-augmented versions of DenseNet-BC-40-48 reported in Table \ref{Tbl:Results} not included in Figure \ref{Fig:ResDense-errorbars} over the course of 200 epochs.}
  \label{Fig:densenet48-apdx}
\end{figure} 

\begin{figure}[H]
  \centering
  \includegraphics[scale=.6]{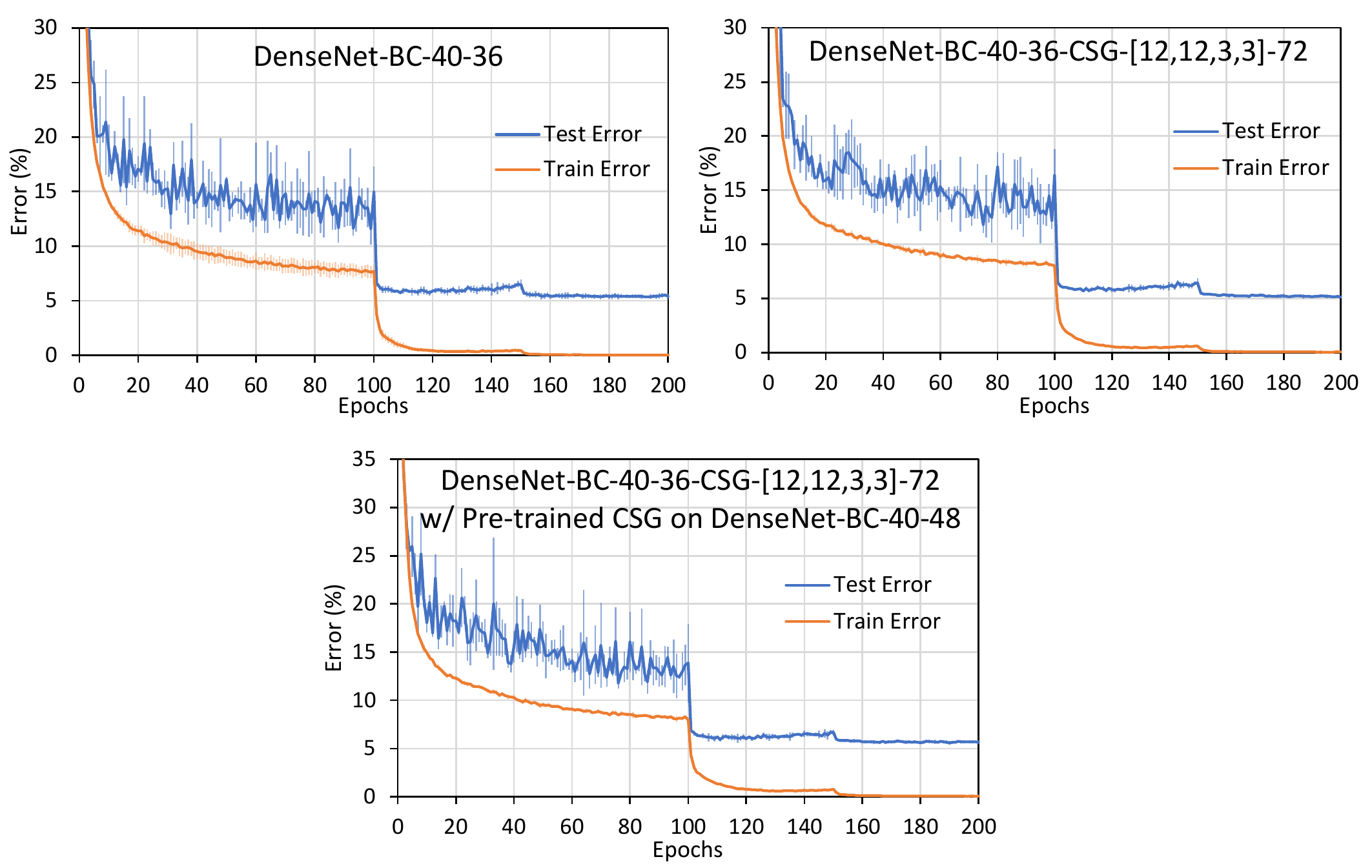}
  \caption{Training and validation error for DenseNet-BC-40-36 and its CSG-augmented versions over the course of 200 epochs.}\label{Fig:densenet36-apdx}
\end{figure} 

\begin{figure}[H]
  \centering
  \includegraphics[scale=.6]{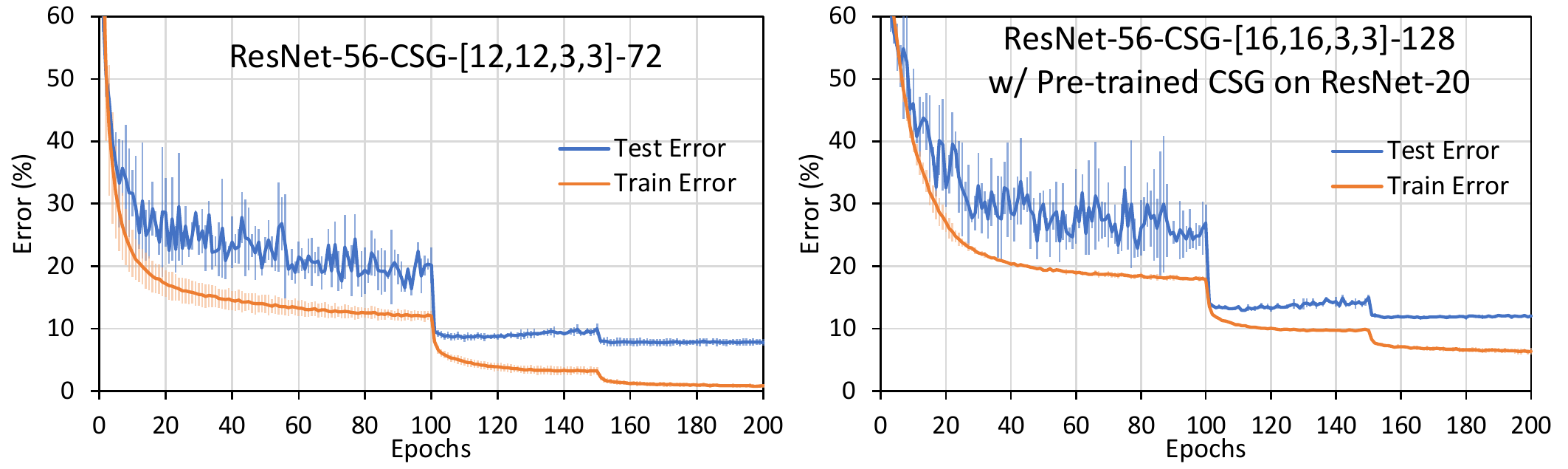}
  \caption{Training and test error for the CSG-augmented versions of DenseNet-BC-40-48 reported in Table \ref{Tbl:Results} not included in Figure \ref{Fig:ResDense-errorbars} over the course of 200 epochs.
  }\label{Fig:resnet-apdx}
\end{figure} 

\section{Training Convergence}\label{App:Proof2}
In this section of the appendix we provide the proof of Theorem~\ref{Thrm:cnn-CSG}.
\begin{proof}[Proof of Theorem~\ref{Thrm:cnn-CSG}]
First of all we note that since the number of weights in the convolutional layer is a polynomial function of $|\hat{\mathcal{C}}|$, it has replaced the $m$ in Theorem~\ref{Thrm:cnn}. Now, let 
\begin{align}
C = \left[c_1,...,c_{\hat{|\mathcal{C}|}}\right],
\end{align}
denote a matrix whose columns are the code vectors corresponding to the slices of the convolutional layer. Now, instead of assuming that the convolutional filter is first generated and then it is used for the convolution operation, equivalently, using associativity, we can assume that each column of the matrix $A_{CSG}$ denotes a vectorized version of a slice of a convolutional filter. It means that, for each column of $A_{CSG}$, we need to calculate the convolution of a slice for its $|\hat{\mathcal{C}}|$ possible locations in the filter. But each of these would be an ordinary convolution with appropriate zero-paddings. Now, the matrix $C$ can be viewed as an additional fully connected layer before the final classification layers. Hence we are dealing with a CNN with an additional fully connected layer at the final stage for which the results in~\cite{allen2018convergence} and specially Theorem~\ref{Thrm:cnn} holds.
\end{proof}

\end{document}